\documentclass[12pt,twoside]{article}
\usepackage{latexsym,amsmath,amssymb}

\newcommand{\eqn}[1]{\begin{align}#1\end{align}}
\newcommand{\eq}[1]{\begin{align*}#1\end{align*}}
\newcommand{\proofof}{Proof of }
\newcommand{\proofend}{}

\def\subsubsect#1{\vspace{1ex plus 0.5ex minus 0.5ex}\noindent{\bf\boldmath{#1.}}}

\newcommand{\argmax}{\operatornamewithlimits{arg\,max}}

\newcommand{\R}[0]{\mathbb R}
\newcommand{\N}[0]{\mathbb N}

\usepackage{amsthm}
\theoremstyle{plain}
\newtheorem{theorem}{Theorem}
\newtheorem{lemma}[theorem]{Lemma}
\theoremstyle{definition}
\newtheorem{definition}[theorem]{Definition}

\theoremstyle{remark}

\newenvironment{keywords}{\centerline{\bf\small
Keywords}\begin{quote}\small}{\par\end{quote}\vskip 1ex}

\newcommand{\M}{\mathbf M}
\newcommand{\m}{\mathbf m}
\newcommand{\Mn}{\mathbf M_{norm}}
\newcommand{\Um}{{U_M}}
\newcommand{\Up}{{U_P}}

\newcommand{\Q}{\mathbb Q}
\newcommand{\B}{\mathcal B}
\newcommand{\leqt}{\stackrel{\times}\leq}
\newcommand{\geqt}{\stackrel{\times}\geq}
\newcommand{\eqt}{\stackrel{\times}=}
\newcommand{\prefix}{\sqsubset}
\newcommand{\prefixeq}{\sqsubseteq}

\begin{document}

\title{
\vskip 2mm\bf\Large\hrule height5pt \vskip 4mm
Universal Prediction of Selected Bits
\vskip 4mm \hrule height2pt}
\author{{\bf Tor Lattimore} and {\bf Marcus Hutter} and {\bf Vaibhav Gavane}\\[3mm]
\normalsize Australian National University \\[-0.5ex]
\normalsize\texttt{tor.lattimore@anu.edu.au} \\
\normalsize Australian National University and ETH Z{\"u}rich \\[-0.5ex]
\normalsize\texttt{marcus.hutter@anu.edu.au} \\
\normalsize VIT University, Vellore \\[-0.5ex]
\normalsize\texttt{vaibhav.gavane@gmail.com}
}
\date{20 July 2011}

\maketitle

\begin{abstract}
Many learning tasks can be viewed as sequence prediction problems. For example, online classification can be converted
to sequence prediction with the sequence being pairs
of input/target data and where the goal is to correctly predict the target data given input data and previous input/target pairs.
Solomonoff induction is known to solve the
general sequence prediction problem, but only if the entire sequence is sampled from a computable distribution. In the case of classification
and discriminative learning though,
only the targets need be structured (given the inputs). We show that the normalised version of Solomonoff induction can still be used in this case,
and more generally that it can detect any recursive sub-pattern (regularity) within an otherwise completely unstructured sequence. It is also shown
that the unnormalised version can fail to predict very simple recursive sub-patterns.
\def\contentsname{\centering\normalsize Contents}
{\parskip=-2.7ex\tableofcontents}
\end{abstract}


\newpage
\section{Introduction}

The sequence prediction problem is the task of predicting the next symbol, $x_n$ after observing $x_1 x_2 \cdots x_{n-1}$.
Solomonoff induction \cite{Sol64a,Sol64b} solves this problem by taking inspiration from Occam's razor and Epicurus' principle of multiple
explanations. These ideas are formalised in the field of Kolmogorov complexity, in particular by the universal a priori semi-measure $\M$.

Let $\mu(x_n|x_1\cdots x_{n-1})$ be the true (unknown) probability of seeing $x_n$ having already observed $x_1\cdots x_{n-1}$.
The celebrated result of Solomonoff \cite{Sol64a} states that if $\mu$ is computable then
\eqn{\label{eqn-sol}
\lim_{n\to\infty} \left[\M(x_n|x_1\cdots x_{n-1}) - \mu(x_n | x_1\cdots x_{n-1})\right] = 0 \text{ with } \mu\text{-probability } 1
}
That is, $\M$
can learn the true underlying distribution from which the data is sampled with probability 1. Solomonoff induction is arguably the gold
standard predictor, universally solving many (passive) prediction problems \cite{Hut04,Hut07, Sol64a}.

However, Solomonoff induction makes no guarantees if $\mu$ is not computable. This would not be problematic if it were unreasonable to
predict sequences sampled from incomputable $\mu$, but this is not the case. Consider the sequence below, where every
even bit is the same as the preceding odd bit, but where the odd bits may be chosen arbitrarily.
\eqn{
\label{eqn-intro-1} \text{00 11 11 11 00 11 00 00 00 11 11 00 00 00 00 00 11 11}
}
Any child will quickly learn the pattern that each even bit is the same as the preceding odd bit and will
correctly predict the even bits. If Solomonoff induction is to be considered a truly intelligent predictor then it too should be
able to predict the even bits. More generally, it should be able to detect any computable sub-pattern.
It is this question, first posed in \cite{Hut04,Hut09open} and resisting attempts by experts for 6 years, that we address.

At first sight, this appears to be an esoteric question, but consider the following problem. Suppose you are given a sequence of pairs,
$x_1 y_1 x_2 y_2 x_3 y_3\cdots$
where $x_i$ is the data for an image (or feature vector) of a character and $y_i$ the corresponding ascii code (class label) for that character.
The goal of online classification is to construct a predictor
that correctly predicts $y_i$ given $x_i$ based on the previously seen training pairs. It is reasonable to assume that there is a relatively
simple pattern to generate $y_i$ given $x_i$ (humans and computers seem to find simple patterns for character recognition). However it is not necessarily reasonable to assume there exists a simple, or even computable,
underlying distribution generating the training data $x_i$. This problem is precisely what gave rise to discriminative learning \cite{LR06}.

It turns out that there exist sequences with even bits equal to preceding odd bits on which the conditional distribution of $\M$
fails to converge to $1$ on the even bits.
On the other hand, it is known that $\M$ is a defective measure, but may be normalised to a proper measure, $\Mn$.
We show that this normalised version {\it does} converge on any recursive sub-pattern
of any sequence, such as that in Equation (\ref{eqn-intro-1}). This outcome is unanticipated since (all?) other results in the field
are independent of normalisation \cite{Hut04,Hut07,LV08, Sol64a}. The proofs are completely different to the standard proofs of
predictive results.

\section{Notation and Definitions}

We use similar notation to \cite{Gacs83,Gacs08,Hut04}. For a more comprehensive introduction to Kolmogorov complexity and
Solomonoff induction see \cite{Hut04,Hut07, LV08, ZL70}.

\subsubsect{Strings}
A finite binary string $x$ is a finite sequence $x_1x_2x_3\cdots x_n$ with $x_i \in \B = \left\{0, 1\right\}$. Its length is denoted $\ell(x)$.
An infinite binary string $\omega$ is an infinite sequence $\omega_1\omega_2\omega_3\cdots$.
The empty string of length zero is denoted $\epsilon$.
$\B^n$ is the set of all binary strings of length $n$. $\B^*$ is the set of all finite binary strings. $\B^\infty$ is the
set of all infinite binary strings.
Substrings are denoted $x_{s:t} := x_s x_{s+1}\cdots x_{t-1}x_{t}$ where $s,t \in \N$ and $s \leq t$. If $s > t$ then $x_{s:t} = \epsilon$.
A useful shorthand is $x_{<t} := x_{1:t-1}$.
Strings may be concatenated.
Let $x,y \in \B^*$ of length $n$ and $m$ respectively. Let $\omega \in \B^\infty$. Then,
\eq{
xy &:= x_1x_2\cdots x_{n-1}x_{n}y_1y_2\cdots y_{m-1}y_m \\
x\omega &:= x_1x_2\cdots x_{n-1}x_{n}\omega_1\omega_2\omega_3\cdots
}
For $b \in \B$, $\neg b = 0$ if $b = 1$ and $\neg b = 1$ if $b = 0$.
We write
$x \prefixeq y$ if $x$ is a prefix of $y$. Formally, $x \prefixeq y$ if $\ell(x) \leq \ell(y)$ and $x_i = y_i$ for all $1 \leq i \leq \ell(x)$.
$x \prefix y$ if $x \prefixeq y$ and $\ell(x) < \ell(y)$.

\subsubsect{Complexity}
Here we give a brief introduction to Kolmogorov complexity and the associated notation.

\begin{definition}[Inequalities]
Let $f, g$ be real valued functions. We write
$f(x) \geqt g(x)$
if there exists a constant $c > 0$ such that $f(x) \geq c \cdot g(x)$ for all $x$. $f(x) \leqt g(x)$ is defined similarly.
$f(x) \eqt g(x)$ if $f(x) \leqt g(x)$ and $f(x) \geqt g(x)$.
\end{definition}

\begin{definition}[Measures]
We call $\mu:\B^* \to [0,1]$ a {\it semimeasure} if $\mu(x) \geq \sum_{b\in\B} \mu(xb)$ for all $x \in \B^*$, and a probability measure if equality
holds and $\mu(\epsilon) = 1$. $\mu(x)$ is the $\mu$-probability that a sequence starts with $x$. $\mu(b | x) := {\mu(xb) \over \mu(x)}$ is the
probability of observing $b \in \B$ given that $x \in \B^*$ has already been observed.
A function $P:\B^* \to [0,1]$ is a {\it semi-distribution} if $\sum_{x\in\B^*} P(x) \leq 1$ and a probability distribution if equality holds.
\end{definition}

\begin{definition}[Enumerable Functions]
A real valued function $f:A \to \R$ is {\it enumerable} if there
exists a computable function $f:A\times\N \to \Q$ satisfying
$\lim_{t\to\infty} f(a, t) = f(a)$ and $f(a, t+1) \geq f(a, t)$ for all $a \in A$ and $t \in \N$.
\end{definition}

\begin{definition}[Machines]
A Turing machine $L$ is a recursively enumerable set (which may be finite) containing pairs of finite binary strings $(p^1, y^1), (p^2, y^2), (p^3, y^3),\cdots$.

$L$ is a {\it prefix machine} if the set $\left\{p^1, p^2,p^3\cdots\right\}$ is prefix free (no program is a prefix of any other).
It is a {\it monotone machine} if for all $(p, y), (q, x) \in L$ with $\ell(x) \geq \ell(y)$, $p \prefixeq q \implies y \prefixeq x$.

We define $L(p)$ to be the set of strings output by program $p$. This is different for monotone and prefix machines.
For prefix machines, $L(p)$ contains only one element,  $y \in L(p)$ if $(p, y) \in L$. For monotone machines, $y \in L(p)$ if there
exists $(p, x) \in L$ with $y \prefixeq x$ and there does not exist a $(q, z) \in L$ with $q \prefix p$ and $y \prefixeq z$.
For both machines $L(p)$ represents the output of machine $L$ when given input $p$.
If $L(p)$ does not exist then we say $L$ does not halt on input $p$.
Note that for monotone machines it is possible for the same program to output multiple strings. For example $(1, 1), (1, 11), (1, 111), (1, 1111), \cdots$
is a perfectly legitimate monotone Turing machine. For prefix machines this is not possible. Also note that if $L$ is a monotone machine
and there exists an $x \in \B^*$ such that $x_{1:n} \in L(p)$ and $x_{1:m} \in L(p)$ then $x_{1:r} \in L(p)$ for all $n \leq r \leq m$.
\end{definition}

\begin{definition}[Complexity] Let $L$ be a prefix or monotone machine then define
\eq{
\lambda_L(y) &:= \sum_{p : y \in L(p)} 2^{-\ell(p)} & C_L(y) &:= \min_{p \in \B^*} \left\{ \ell(p) : y \in L(p) \right\}
}
If $L$ is a prefix machine then we write $\m_L(y) \equiv \lambda_L(y)$. If $L$ is a monotone machine then we write $\M_L(y) \equiv \lambda_L(y)$.
Note that if $L$ is a prefix machine then $\lambda_L$ is an enumerable semi-distribution while if $L$ is a monotone machine, $\lambda_L$ is
an enumerable semi-measure. In fact, every enumerable semi-measure (or semi-distribution) can be represented via some machine $L$ as $\lambda_L$.
\end{definition}
For prefix/monotone machine $L$ we write $L_t$ for the first $t$ program/output pairs in the recursive
enumeration of $L$, so $L_t$ will be a finite set containing at most $t$ pairs.\footnote{$L_t$ will contain exactly $t$ pairs unless $L$ is finite,
in which case it will contain $t$ pairs until $t$ is greater than the size of $L$. This annoyance will never be problematic.}

The set of all monotone (or prefix) machines is itself recursively enumerable \cite{LV08},\footnote{Note the
enumeration may include repetition, but this is unimportant in this case.} which allows one to define a
universal monotone machine $\Um$
as follows. Let $L^i$ be the $i$th monotone machine in the recursive enumeration of monotone machines.
\eq{
(i'p, y) \in \Um \Leftrightarrow (p, y) \in L^i
}
where $i'$ is a prefix coding of the integer $i$. A universal prefix machine, denoted $\Up$, is defined in a similar way. For details see \cite{LV08}.

\begin{theorem}[Universal Prefix/Monotone Machines]\label{thm_universal}
For the universal monotone machine $\Um$ and universal prefix machine $\Up$,
\eq{
\m_\Up(y) &> c_L \m_L(y) \text{ for all } y\in \B^* & \M_\Um(y) &> c_L \M_L(y) \text{ for all } y\in\B^*
}
where $c_L > 0$ depends on $L$ but not $y$.
\end{theorem}
For a proof, see \cite{LV08}. As usual, we will fix reference universal prefix/monotone machines $\Up$, $\Um$ and drop the subscripts by letting,
\eq{
\m(y) &:= \m_\Up(y) \equiv \sum_{p : y \in \Up(p)} 2^{-\ell(p)} & \M(y) &:= \M_\Um(y) \equiv \sum_{p : y\in \Um(p)} 2^{-\ell(p)} \\
K(y) &:= C_\Up(y) \equiv \min_{p \in \B^*} \left\{\ell(p) : y\in\Up(p)\right\} & Km(y) &:= \min_{p\in\B^*} \left\{\ell(p) : y\in\Um(p)\right\}
}
The choice of reference universal Turing machine is usually\footnote{See \cite{HM07} for a subtle exception. All the results in this paper are independent of universal Turing machine.} unimportant since a different choice varies $\m, \M$ by only a multiplicative
constant, while $K, Km$ are varied by additive constants.
For natural numbers $n$ we define $K(n)$ by $K(\left<n\right>)$ where $\left<n\right>$ is the binary representation of $n$.

$\M$ is not a proper measure, $\M(x) > \M(x0) + \M(x1)$ for all $x \in \B^*$, which means that $\M(0|x) + \M(1|x) < 1$, so $\M$
assigns a non-zero probability that the sequence will end. This is because there are monotone programs $p$ that halt, or enter infinite loops.
For this reason Solomonoff introduced a normalised version, $\Mn$ defined as follows.
\begin{definition}[Normalisation]
\eq{
\Mn(\epsilon) &:=1 & \Mn(y_n | y_{<n}) \equiv {\Mn(y_{1:n}) \over \Mn(y_{<n})}:= {\M(y_{1:n}) \over \M(y_{<n}0) + \M(y_{<n}1)}.
}
\end{definition}
This normalisation is not unique, but is philosophically and technically the most attractive and was used and defended by Solomonoff.
Historically, most researchers have accepted the defective $\M$ for technical convenience. As mentioned, the difference seldom matters,
but in this paper it is somewhat surprisingly crucial. For a discussion of normalisation, see \cite{LV08}.

\begin{theorem}\label{thm_kolmogorov}The following are results in Kolmogorov complexity. Proofs for all can be found in \cite{LV08}.
\begin{enumerate}
\item $\m(x) \eqt 2^{-K(x)}$
\item $2^{-K(xb)} \eqt 2^{-K(x\neg b)}$
\item $\M(x) \geqt \m(x)$
\item If $P$ is an enumerable semi-distribution, then $\m(y) \geqt P(y)$
\item If $\mu$ is an enumerable semi-measure, then $\M(y) \geqt \mu(y)$
\end{enumerate}
\end{theorem}
Note the last two results are equivalent to Theorem \ref{thm_universal} since every enumerable semi-(measure/distribution) is
generated by a monotone/prefix machine in the sense of Theorem \ref{thm_universal} and vice-versa.

Before proceeding to our own theorems we need a recently proven result in algorithmic information theory.
\begin{theorem}\label{thm_gap}[Lempp, Miller, Ng and Turetsky, 2010, unpublished, private communication]
$\lim_{n\to\infty} {\m(\omega_{<n}) \over \M(\omega_{<n})} = 0$, for all $\omega \in \B^\infty$.
\end{theorem}

\section{$\Mn$ Predicts Selected Bits}

The following Theorem is the main positive result of this paper. It shows that any computable sub-pattern of a sequence will
eventually be predicted by $\Mn$.
\begin{theorem}\label{thm_positive}
Let $f:\B^*\to \B \cup \left\{\epsilon\right\}$ be a total recursive function and $\omega \in \B^\infty$ satisfying $f(\omega_{<n}) = \omega_n$
whenever $f(\omega_{<n}) \neq \epsilon$.
If $f(\omega_{<n_i}) \neq \epsilon$ is defined for an infinite sequence $n_1, n_2, n_3, \cdots$ then
$\lim_{i\to\infty} \Mn(\omega_{n_i} | \omega_{<n_i}) = 1$.
\end{theorem}
Essentially the Theorem is saying that if there exists a computable predictor $f$ that correctly predicts the next bit every time it tries
(i.e when $f(\omega_{<n}) \neq \epsilon$) then $\Mn$ will eventually predict the same bits as $f$.
By this we mean that if you constructed a predictor $f_{\Mn}$ defined by $f_{\Mn}(\omega_{<n}) = \argmax_{b\in\B} \Mn(b|\omega_{<n})$, then
there exists an $N$ such that $f_{\Mn}(\omega_{<n}) = f(\omega_{<n})$ for all $n > N$ where $f(\omega_{<n}) \neq \epsilon$.
For example, let $f$ be defined by
\eq{
f(x) = \begin{cases}
x_{\ell(x)} & \text{if } \ell(x) \text{ odd} \\
\epsilon & \text{otherwise}
\end{cases}
}
Now if $\omega \in \B^\infty$ satisfies $\omega_{2n} = f(\omega_{<2n}) = \omega_{2n-1}$ for all $n \in \N$
then Theorem \ref{thm_positive} shows that $\lim_{n\to\infty} \Mn(\omega_{2n}|\omega_{<2n}) = 1$. It says nothing
about the predictive qualities of $\Mn$ on the odd bits, on which there are no restrictions.

The proof essentially relies on using $f$ to show that monotone programs for $\omega_{<n_i} \neg\omega_{n_i}$ can be converted to prefix programs.
This is then used to show that $\M(\omega_{<n_i} \neg\omega_{n_i}) \eqt \m(\omega_{<n_i}\neg\omega_{n_i})$. The result will then
follow from Theorem \ref{thm_gap}.

Theorem \ref{thm_positive} insists that $f$ be totally recursive and that $f(\omega_{<n}) = \epsilon$ if $f$ refrains
from predicting. One could instead allow $f$ to be partially recursive and simply not halt to avoid making a prediction. The proof below breaks
down in this case and we suspect that Theorem \ref{thm_positive} will become invalid if $f$ is permitted to be only partially recursive.

\begin{proof}[\proofof Theorem \ref{thm_positive}]
We construct a machine $L$ from $\Um$ consisting of all programs that produce output that $f$ would not predict. We then show that
these programs essentially form a prefix machine. Define $L$ by the following process
\begin{enumerate}
\item $L := \emptyset$ and $t := 1$.
\item Let $(p, y)$ be the $t$th pair in $\Um$.
\item Let $i$ be the smallest natural number such that $y_i \neq f(y_{<i}) \neq \epsilon$. That is, $i$ is the position at which $f$ makes its first
mistake when predicting $y$. If $f$ makes no prediction errors then $i$ doesn't exist.\footnote{This is where
the problem lies for partially recursive prediction functions. Computing the smallest $i$ for which $f$ predicts incorrectly is incomputable
if $f$ is only partially recursive, but computable if it is totally recursive. It is this distinction that allows $L$ to be recursively enumerable, and so
be a machine.}
\item If $i$ exists then $L := L \cup \left\{(p, y_{1:i}) \right\}$ (Note that we do not allow $L$ to contain duplicates).
\item $t := t+1$ and go to step 2.
\end{enumerate}
Since $f$ is totally recursive and $\Um$ is recursively enumerable, the process above shows that $L$ is recursively enumerable. It is easy
to see that $L$ is a monotone machine. Further, if $(p, y), (q, x) \in L$ with $p \prefixeq q$ then $y = x$. This follows since
by monotonicity we would have that $y \prefixeq x$, but $f(x_{<\ell(y)}) = f(y_{<\ell(y)}) \neq y_{\ell(y)} = x_{\ell(y)}$ and by steps 3 and 4 in
the process above we have that $\ell(x) = \ell(y)$.

Recall that $L_t$ is the $t$th enumeration of $L$ and contains $t$ elements. Define $\bar L_t \subseteq L_t$
to be the largest prefix free set of shortest programs. Formally, $(p, y) \in \bar L_t$ if there does not exist a $(q, x) \in L_t$
such that $q \prefix p$. For example, if $L_t = (1, 001), (11, 001), (01, 11110), (010, 11110)$ then $\bar L_t = (1, 001), (01, 11110)$. If
we now added $(0, 11110)$ to $L_t$ to construct $L_{t+1}$ then $\bar L_{t+1}$ would be $(1, 001), (0, 11110)$.

Since $L_t$ is finite, $\bar L_t$ is easily computable from $L_t$. Therefore the following function is computable.
\eq{
P(y, t) := \sum_{p : (p, y) \in \bar L_t} 2^{-\ell(p)} \geq 0.
}
Now $\bar L_t$ is prefix free, so by Kraft's inequality $\sum_{y\in\B^*} P(y, t) \leq 1$ for all $t \in \N$.
We now show that $P(y, t + 1) \geq P(y, t)$ for all $y \in \B^*$ and $t\in \N$ which proves that $P(y) = \lim_{t\to\infty} P(y, t)$ exists
and is a semi-distribution.

Let $(p, y)$ be the program/output pair in $L_{t+1}$ but not in $L_t$. To see how $P(\cdot, t)$ compares to $P(\cdot, t+1)$ we need to
compare $\bar L_{t}$ and $\bar L_{t+1}$. There are three cases:
\begin{enumerate}
\item There exists a $(q, x) \in L_t$ with $q \prefix p$. In this case $\bar L_{t+1} = \bar L_t$.
\item There does not exist a $(q, x) \in L_t$ such that $p \prefix q$. In this case $(p, y)$ is simply added to
$\bar L_t$ to get $\bar L_{t+1}$ and so $\bar L_t \subset \bar L_{t+1}$. Therefore $P(\cdot, t+1) \geq P(\cdot, t)$ is clear.
\item There does exist a $(q, x) \in \bar L_t$ such that $p \prefix q$. In this case $\bar L_{t+1}$ differs from $\bar L_t$ in that
it contains $(p, y)$ but not $(q, x)$. Since $p \prefix q$ we have that $y = x$. Therefore
$P(y, t+1) - P(y, t) = 2^{-\ell(p)} - 2^{-\ell(q)} > 0$ since $p \prefix q$. For other values, $P(\cdot, t) = P(\cdot, t+1)$.
\end{enumerate}
Note that it is not possible that $p = q$ since then $x = y$ and duplicates are not added to $L$.
Therefore $P$ is an enumerable semi-distribution. By Theorem \ref{thm_kolmogorov} we have
\eqn{
\label{eqn-pos1} \m(\omega_{<n_i}\neg\omega_{n_i}) \geqt P(\omega_{<n_i}\neg\omega_{n_i})
}
where the constant multiplicative fudge factor in the $\geqt$ is independent of $i$.
Suppose $\omega_{<n_i}\neg\omega_{n_i} \in \Um(p)$.
Therefore there exists a $y$ such that $\omega_{<n_i}\neg\omega_{n_i} \prefixeq y$ and $(p, y)\in\Um$. By parts 2 and 3 of the process above,
$(p, \omega_{<n_i}\neg\omega_{n_i})$ is added to $L$. Therefore there exists a $T \in \N$
such that $(p, \omega_{<n_i}\neg\omega_{n_i}) \in L_t$ for all $t \geq T$.

Since $\omega_{<n_i}\neg\omega_{n_i} \in \Um(p)$, there does not exist a $q \prefix p$ with $\omega_{<n_i}\neg\omega_{n_i} \in \Um(q)$. Therefore
eventually, $(p, \omega_{<n_i}\neg\omega_{n_i}) \in \bar L_t$ for all $t \geq T$. Since every program in $\Um$ for $\omega_{<n_i}\neg\omega_{n_i}$
is also a program in $L$, we get
\eq{
\lim_{t\to\infty} P(\omega_{<n_i}\neg\omega_{n_i}, t) \equiv P(\omega_{<n_i}\neg\omega_{n_i}) = \M(\omega_{<n_i} \neg\omega_{n_i}).
}
Next,
\eqn{
\label{eqn-pos2} \Mn(\neg\omega_{n_i} | \omega_{<n_i}) &\equiv {\M(\omega_{<n_i}\neg\omega_{n_i}) \over \M(\omega_{<n_i}\omega_{n_i}) + \M(\omega_{<n_i}\neg\omega_{n_i})} \\
\label{eqn-pos3} &\leqt {\m(\omega_{<n_i}\neg\omega_{n_i}) \over \M(\omega_{1:n_i})} \\
\label{eqn-pos4} &\eqt {\m(\omega_{1:n_i}) \over \M(\omega_{1:n_i})}
}
where Equation (\ref{eqn-pos2}) follows by the definition of $\Mn$. Equation (\ref{eqn-pos3}) follows from Equation (\ref{eqn-pos1}) and algebra. Equation (\ref{eqn-pos4}) follows since $\m(xb) \eqt 2^{-K(xb)} \eqt 2^{-K(x\neg b)} \eqt \m(x\neg b)$, which is Theorem \ref{thm_kolmogorov}.
However, by Theorem \ref{thm_gap}, $\lim_{i\to\infty} {\m(\omega_{<n_i}) \over \M(\omega_{<n_i})} = 0$ and so $\lim_{i\to\infty} \Mn(\neg\omega_{n_i} | \omega_{<n_i}) = 0$. Therefore $\lim_{i\to\infty} \Mn(\omega_{n_i} | \omega_{<n_i}) = 1$
as required.
\proofend\end{proof}
We have remarked already that Theorem \ref{thm_positive} is likely not valid if $f$ is permitted to be a partial recursive function that only
output on sequences for which they make a prediction. However, there is a class of predictors larger than the totally recursive ones of Theorem
\ref{thm_positive}, which $\Mn$ still learns.
\begin{theorem}\label{thm_positive2}
Let $f:\B^*\to \B \cup \left\{\epsilon\right\}$ be a partial recursive function and $\omega \in \B^\infty$ satisfying
\begin{enumerate}
\item $f(\omega_{<n})$ is defined for all $n$.
\item $f(\omega_{<n}) = \omega_n$ whenever $f(\omega_{<n}) \neq \epsilon$.
\end{enumerate}
If $f(\omega_{<n_i}) \in \B$ for an infinite sequence $n_1, n_2, n_3, \cdots$ then
\eq{
\lim_{i\to\infty} \Mn(\omega_{n_i} | \omega_{<n_i}) = 1.
}
\end{theorem}
The difference between this result and Theorem \ref{thm_positive} is that $f$ need only be defined on all prefixes of at least one $\omega \in \B^\infty$
and not everywhere in $\B^*$. This allows for a slightly broader class of predictors. For example, let $\omega = p^1 b^1 p^2 b^2 p^3 b^3 \cdots$ where
$p^i$ is some prefix machine that outputs at least one bit and $b^i$ is the first bit of that output. Now there exists a computable $f$
such that $f(p^1 b^1 \cdots p^{i-1} b^{i-1} p^i) = b^i$ for all $i$ and $f(\omega_{<n}) = \epsilon$ whenever $\omega_n \neq b^i$ for some $i$ ($f$
only tries to predict the outputs). By Theorem \ref{thm_positive2}, $\Mn$ will correctly predict $b^i$.

The proof of Theorem \ref{thm_positive2} is almost identical to that of Theorem \ref{thm_positive}, but with one additional subtlety.\\
{\it Proof sketch. }
The proof follows that of Theorem \ref{thm_positive} until the construction of $L$. This breaks down because step 3
is no longer computable since $f$ may not halt on some string that is not a prefix of $\omega$. The modification is to run steps 2-4
in parallel for all $t$ and only adding $(p, y_{1:i})$ to $L$ once it has been proven that $f(y_{<i}) \neq y_i$ and $f(y_{<k})$ halts for all $k < i$,
and either chooses not to predict (outputs $\epsilon$), or predicts correctly. Since $f$ halts on all
prefixes of $\omega$, this does not change $L$ for any programs we care about and the remainder of the proof goes through identically. \proofend

It should be noted that this new class of predictors is still less general than allowing $f$ to an arbitrary partial recursive predictor. For
example, a partial recursive $f$ can predict
the ones of the halting sequence, while choosing not to predict the zeros (the non-halting programs). It is clear this cannot be modified into
a computable $f$ predicting both ones and zeros, or predicting ones and outputting $\epsilon$ rather than zero, as this would solve the halting problem.

\section{$\M$ Fails to Predict Selected Bits}

The following theorem is the corresponding negative result that while the conditional distribution of $\Mn$ converges to $1$ on
recursive sub-patterns, $\M$ can fail to do so.
\begin{theorem}\label{thm_negative}
Let $f:\B^*\to\B\cup\left\{\epsilon\right\}$ be the total recursive function defined by,
\eq{
f(z) := \begin{cases}
z_{\ell(z)} & \text{if } \ell(z) \text{ odd} \\
\epsilon & \text{otherwise}
\end{cases}
}
There exists an infinite string $\omega \in \B^\infty$ with $\omega_{2n} = f(\omega_{<2n}) \equiv \omega_{2n-1}$ for all $n \in \N$ such that
\eq{
\liminf_{n\to\infty} \M(\omega_{2n} | \omega_{<2n}) < 1.
}
\end{theorem}
The proof requires some lemmas.
\begin{lemma}\label{lem_M}$\M(xy)$ can be bounded as follows.
\eqn{
\label{eqn-lemM} 2^{K(\ell(x))}\M(y) \geqt \M(xy) &\geqt \M(y)2^{-K(x)}.
}
\end{lemma}
\begin{proof}
Both inequalities are proven relatively easily by normal methods as used in \cite{LV08} and elsewhere. Nevertheless we present them as a warm-up
to the slightly more subtle proof later.

Now construct monotone machine $L$, which we should think of as taking two programs as input. The first, a prefix program $p$, the output of
which we view as a natural number $n$. The second, a monotone program. We then simulate the monotone machine and strip the first $n$ bits of
its output. $L$ is formally defined as follows.
\begin{enumerate}
\item $L := \emptyset$, $t := 1$
\item Let $(p, n), (q, y)$ be the $t$th pair of program/outputs in $\Up\times \Um$, which is enumerable.
\item If $\ell(y) \geq n$ then add $(pq, y_{n+1:\ell(y)})$ to $L$
\item $t := t+1$ and go to step 2
\end{enumerate}
By construction, $L$ is enumerable and is a monotone machine. Note that if $xy \in \Um(q)$ and $\ell(x) \in \Up(p)$ then $y \in L(pq)$.
Now,
\eqn{
\label{eqn-lemM-1} \M(y) \geqt \M_L(y) &\equiv \sum_{r : y \in L(r)} 2^{-\ell(r)} \geq \sum_{q, p : xy \in \Um(q), \ell(x) \in \Up(p)} 2^{-\ell(pq)} \\
\label{eqn-lemM-2} &=\sum_{q : xy \in \Um(q)} 2^{-\ell(q)} \sum_{p : \ell(x) \in \Up(p)} 2^{-\ell(p)} \equiv \M(xy) \m(\ell(x)) \\
\label{eqn-lemM-3} &\eqt \M(xy) 2^{-K(\ell(x))}
}
where Equation (\ref{eqn-lemM-1}) follows by Theorem \ref{thm_universal}, definitions and because if $xy \in \Um(q)$ and $\ell(x) \in \Up(p)$ then $y \in L(pq)$.
Equation (\ref{eqn-lemM-2}) by algebra, definitions. Equation (\ref{eqn-lemM-3}) by Theorem \ref{thm_kolmogorov}.

The second inequality is proved similarly. We define a machine $L$ as follows,
\begin{enumerate}
\item $L = \emptyset$, $t:=1$
\item Let $(q, x), (r, y)$ be the $t$th element in $\Up\times \Um$, which is enumerable.
\item Add $(qr, xy)$ to $L$
\item $t:=t+1$ and go to step 2
\end{enumerate}
It is easy to show that $L$ is monotone by using the properties of $\Up$ and $\Um$. Now,
\eq{
\M(xy) \geqt \M_L(xy) &\equiv \sum_{p : xy \in L(p)} 2^{-\ell(p)} \geq \sum_{q, r : x \in \Up(q), y \in \Um(r)} 2^{-\ell(qr)}  \\
&= \sum_{q : x \in \Up(q)} 2^{-\ell(q)} \sum_{r : y \in \Um(r)} 2^{-\ell(r)} \equiv \m(x) \M(y) \eqt 2^{-K(x)} \M(y).
}
\proofend\end{proof}
\begin{lemma}\label{lem_gap}
There exists an $\omega \in \B^\infty$ such that
\eq{
\liminf_{n\to\infty} \left[\M(0|\omega_{<n}) + \M(1|\omega_{<n}) \right] = 0.
}
\end{lemma}
\begin{proof}
First we show that for each $\delta > 0$ there exists a $z \in \B^*$ such that $\M(0|z) + \M(1|z) < \delta$. This result is already known
and is left as an exercise (4.5.6) with a proof sketch in \cite{LV08}. For completeness, we include a proof. Recall that $\M(\cdot, t)$ is
the function approximating $\M(\cdot)$ from below. Fixing an $n$, define $z \in \B^*$ inductively as follows.
\begin{enumerate}
\item $z := \epsilon$
\item Let $t$ be the first natural number such that $\M(zb, t) > 2^{-n}$ for some $b \in \B$.
\item If $t$ exists then $z := z\neg b$ and repeat step 2. If $t$ does not exist then $z$ is left unchanged (forever).
\end{enumerate}
Note that $z$ must be finite since each time it is extended, $\M(zb, t) > 2^{-n}$. Therefore $\M(z\neg b, t) < \M(z, t) - 2^{-n}$ and so
each time $z$ is extended, the value of $\M(z, t)$ decreases by at least $2^{-n}$ so eventually $\M(zb, t) < 2^{-n}$ for all $b \in \B$. Now once the $z$
is no longer being extended ($t$ does {\it not} exist in step 3 above) we have
\eqn{
\label{eqn-lem1} \M(z0) + \M(z1) &\leq 2^{1 - n}.
}
However we can also show that $\M(z) \geqt 2^{-K(n)}$. The intuitive idea is that the process above requires only the value of $n$, which can be
encoded in $K(n)$ bits. More formally, let $p$ be such that $n \in \Up(p)$ and note that the following set is
recursively enumerable (but not recursive) by the process above.
\eq{
L_{p} := (p, \epsilon), (p, z_{1:1}), (p, z_{1:2}), (p, z_{1:3}), \cdots, (p, z_{1:\ell(z) - 1}), (p, z_{1:\ell(z)}).
}
Now take the union of all such sets, which is a) recursively enumerable since $\Up$ is, and b) a monotone machine because $\Up$ is a prefix machine.
\eq{
L := \bigcup_{(p, n) \in \Up} L_{p}.
}
Therefore
\eqn{
\label{eqn-lem2} \M(z) \geqt \M_L(z) \geq 2^{-K(n)}
}
where the first inequality is from Theorem \ref{thm_universal} and the second follows since if $n^*$ is the program of length $K(n)$
with $\Up(n^*) = n$ then $(n^*, z_{1:\ell(z)}) \in L$. Combining Equations (\ref{eqn-lem1}) and (\ref{eqn-lem2}) gives
\eq{
\M(0|z) + \M(1|z) &\leqt 2^{1 - n + K(n)}.
}
Since this tends to zero as $n$ goes to infinity,\footnote{An integer $n$ can easily be encoded in $2 \log n$ bits, so $K(n) \leq 2 \log n + c$ for some $c > 0$ independent of $n$.} for each $\delta > 0$ we can construct a $z\in\B^*$ satisfying $\M(0|z) + \M(1|z) < \delta$, as required.
For the second part of the proof, we construct $\omega$ by concatenation.
\eq{\omega := z^1 z^2z^3 \cdots
}
where $z^n \in \B^*$ is chosen such that,
\eqn{
\label{eqn-lem0} \M(0|z^n) + \M(1|z^n) < \delta_n
}
with $\delta_n$ to be chosen later. Now,
\eqn{
\label{eqn-lem3} \M(b|z^1\cdots z^n) &\equiv {\M(z^1\cdots z^n b) \over \M(z^1\cdots z^n) } \\
\label{eqn-lem4} &\leqt \left[{2^{K(\ell(z^1\cdots z^{n-1})) + K(z^1\cdots z^{n-1})} }\right]{\M(z^n b) \over \M(z^n)} \\
\label{eqn-lem5} &\equiv \left[2^{K(\ell(z^1\cdots z^{n-1})) + K(z^1\cdots z^{n-1})}  \right]  \M(b|z^n)
}
where Equation (\ref{eqn-lem3}) is the definition of conditional probability. Equation (\ref{eqn-lem4}) follows by
applying Lemma \ref{lem_M} with $x = z^1z^2\cdots z^{n-1}$ and $y = z^n$ or $z^n b$.
Equation (\ref{eqn-lem5}) is again the definition of conditional probability. Now let
\eq{
\delta_n = {2^{-n} \over 2^{K(\ell(z^1 \cdots z^{n-1})) + K(z^1\cdots z^{n-1})}}.
}
Combining this with Equations (\ref{eqn-lem0}) and (\ref{eqn-lem5}) gives
\eq{
\M(0 | z^1 \cdots z^n) + \M(1|z^1 \cdots z^n) \leqt 2^{-n}.
}
Therefore,
\eq{
\liminf_{n\to\infty} \left[\M(0|\omega_{<n}) + \M(1|\omega_{<n})\right] = 0
}
as required.
\proofend\end{proof}
\begin{proof}[\proofof Theorem \ref{thm_negative}]
Let $\bar\omega\in\B^\infty$ be defined by $\bar \omega_{2n} := \bar\omega_{2n-1} := \omega_n$ where $\omega$ is the string defined in the previous lemma.
Recall $\Um := \left\{(p^1, y^1), (p^2, y^2), \cdots\right\}$ is the universal monotone machine. Define monotone machine $L$ by the
following process,
\begin{enumerate}
\item $L = \emptyset$, $t = 1$
\item Let $(p, y)$ be the $t$th element in the enumeration of $\Um$
\item Add $(p, y_1 y_3 y_5 y_7 \cdots)$ to $L$
\item $t := t+ 1$ and go to step 2.
\end{enumerate}
Therefore if $\bar\omega_{<2n} \in \Um(p)$ then $\omega_{1:n} \in L(p)$. By identical reasoning as elsewhere,
\eqn{
\label{eqn-neg1} \M(\omega_{1:n}) \geqt \M(\bar\omega_{<2n}).
}
In fact, $\M(\omega_{1:n}) \eqt \M(\bar\omega_{<2n})$, but this is unnecessary.
Let $P := \left\{p : \exists b \in \B \text{ s.t } \omega_{1:n}b \in \Um(p)\right\}$ and $Q := \left\{p : \omega_{1:n} \in \Um(p)\right\} \supset P$.
Therefore
\eq{
1 - \M(0 | \omega_{1:n}) - \M(1 | \omega_{1:n}) &= 1 - {{\sum_{p\in P} 2^{-\ell(p)}} \over {\sum_{q \in Q} 2^{-\ell(q)}}} \\
&= {{\sum_{p \in Q - P} 2^{-\ell(p)}} \over {\sum_{q \in Q} 2^{-\ell(q)}}}.
}
Now let $\bar P := \left\{p : \exists b \in \B \text{ s.t } \bar\omega_{<2n}b\in \Um(p)\right\}$ and
$\bar Q := \left\{p : \bar\omega_{<2n} \in \Um(p)\right\}
\supset \bar P$.
Define monotone machine $L$ by the following process
\begin{enumerate}
\item $L = \emptyset$, $t := 1$
\item Let $(p, y)$ be the $t$th program/output pair in $\Um$
\item Add $(p, y_1y_1 y_2 y_2 \cdots y_{\ell(y)-1} y_{\ell(y)-1} y_{\ell(y)})$ to $L$
\item $t := t + 1$ and go to step 2.
\end{enumerate}
Let $p \in Q - P$. Therefore $\omega_{1:n} \in \Um(p)$ and $\omega_{1:n}b\notin \Um(p)$ for any $b \in \B$. Therefore $\bar\omega_{<2n} \in L(p)$
while $\bar\omega_{<2n}b \notin L(p)$ for any $b \in \B$. Now there exists an $i$ such that $L$ is the $i$th machine in the enumeration of
monotone machines, $L^i$.

Therefore, by the definition of the universal monotone machine $\Um$ we have that $\bar\omega_{<2n}b \notin \Um(i'p) = L^i(p)  = L(p) \ni \bar\omega_{<2n}$ and $\Um(i'p) = L(p)$
for any $b \in \B$. Therefore $i'p \in \bar Q - \bar P$ and so,
\eqn{
\label{eqn-neg1.5} \sum_{q \in \bar Q - \bar P} 2^{-\ell(q)} \geq
\sum_{p : i'p \in \bar Q - \bar P} 2^{-\ell(i'p)} \geq
\sum_{p \in Q - P} 2^{-\ell(i'p)} \eqt \sum_{p \in Q - P} 2^{-\ell(p)}.
}
Therefore
\eqn{
\label{eqn-neg2} 1 - \M(0|\bar\omega_{<2n}) - \M(1 | \bar \omega_{<2n}) &\equiv {\sum_{p \in \bar Q - \bar P} 2^{-\ell(p)} \over \M(\bar\omega_{<2n})} \\
\label{eqn-neg3} &\geqt {\sum_{p \in Q- P} 2^{-\ell(p)} \over \M(\omega_{1:n})} \\
\label{eqn-neg4} &\equiv 1 - \M(0|\omega_{1:n}) - \M(1|\omega_{1:n})
}
where Equation (\ref{eqn-neg2}) follows from the definition of $\bar P$, $\bar Q$ and $\M$. Equation (\ref{eqn-neg3}) by (\ref{eqn-neg1.5}) and
(\ref{eqn-neg1}). Equation (\ref{eqn-neg4}) by the definition of $P, Q$ and $\M$.
Therefore by Lemma \ref{lem_gap} we have
\eq{
\limsup_{n\to\infty} \left[ 1 - \M(0|\bar\omega_{<2n}) - \M(1|\bar\omega_{<2n})\right] \geqt
\limsup_{n\to\infty} \left[ 1 - \M(0|\omega_{1:n}) - \M(1|\omega_{1:n})\right] = 1.
}
Therefore $\liminf_{n\to\infty} \M(\bar\omega_{2n}|\bar\omega_{<2n}) < 1$
as required.
\proofend\end{proof}
Note that $\lim_{n\to\infty} \M(\bar\omega_{2n}|\bar\omega_{<2n}) \neq 0$ in fact, one can show that there exists a $c > 0$ such that
$\M(\bar\omega_{2n}|\bar\omega_{<2n}) > c$ for all $n \in \N$. In this sense $\M$ can still be used to predict in the same way as $\Mn$,
but it will never converge as in Equation (\ref{eqn-sol}).

\section{Discussion}

\subsubsect{Summary}
Theorem \ref{thm_positive} shows that if an infinite sequence contains a computable sub-pattern then the normalised universal semi-measure $\Mn$
will eventually predict it. This means that Solomonoff's normalised version of induction is effective in the classification example given
in the introduction. Note that we have only proven the binary case, but expect the proof will go through identically for arbitrary finite
alphabet.

On the other hand, Theorem \ref{thm_negative} shows that plain $\M$ can fail to predict such structure in the sense that the conditional distribution
need not converge to $1$ on the true sequence. This is because it is not a proper measure, and does not converge to one.
These results are surprising since (all?) other predictive results, including Equation (\ref{eqn-sol}) and many others in \cite{Hut04,Hut07,LV08,Sol64a},
do not rely on normalisation.

\subsubsect{Consequences}
We have shown that $\Mn$ can predict recursive structure in infinite strings that are incomputable (even stochastically so).
These results give hope that a Solomonoff inspired algorithm may be effective at online classification, even when the training data is
given in a completely unstructured way. Note that while $\M$ is enumerable and $\Mn$ is only approximable,\footnote{A function $f$ is
approximable if there exists a computable function $f(\cdot, t)$ with $\lim_{t\to\infty} f(\cdot, t) = f(\cdot)$. Convergence need not be monotonic.}
both the conditional distributions are only approximable, which means it is no harder to predict using $\Mn$ than $\M$.

\subsubsect{Open Questions}
A number of open questions were encountered in writing this paper.
\begin{enumerate}
\item Extend Theorem \ref{thm_positive} to the stochastic case where a sub-pattern is generated
stochastically from a computable distribution rather than merely a computable function. It seems likely that a different approach
will be required to solve this problem.
\item Another interesting question is to strengthen the result by proving a convergence rate. It may be possible to prove that under
the same conditions as Theorem \ref{thm_positive} that $\sum_{i=1}^\infty \left[1 - \Mn(\omega_{n_i}|\omega_{<n_i})\right] \leqt K(f)$ where
$K(f)$ is the (prefix) complexity of the predicting function $f$. Again, if this is even possible, it will likely require a different approach.
\item Prove or disprove the validity of Theorem \ref{thm_positive} when the totally recursive prediction function $f$ (or the modified predictor of Theorem \ref{thm_positive2}) is replaced by
a partially recursive function.
\end{enumerate}

\subsubsect{Acknowledgements}
We thank Wen Shao and reviewers for valuable feedback on earlier drafts and the Australian Research Council for support under grant DP0988049.

\newpage

\begin{small}

\end{small}

\appendix
\section{Table of Notation}

\begin{tabular}{|p{1.6cm} | p{12.5cm}|}
\hline
{\bf Symbol} & {\bf Description} \\
$\B$ & Binary symbols, 0 and 1 \\
$\Q$ & Rational numbers \\
$\N$ & Natural numbers \\
$\B^*$ & The set of all finite binary strings \\
$\B^\infty$ & The set of all infinite binary strings \\
$x, y, z$ & Finite binary strings \\
$\omega$ & An infinite binary string \\
$\bar\omega$ & An infinite binary string with even bits equal to preceding odd bits \\
$\ell(x)$ & The length of binary string $x$ \\
$\neg b$ & The negation of binary symbol $b$. $\neg b = 0$ if $b = 1$ and $\neg b = 1$ if $b = 0$ \\
$p, q$ & Programs \\
$\mu$ & An enumerable semi-measure \\
$\M$ & The universal enumerable semi-measure \\
$\Mn$ & The normalised version of the universal enumerable semi-measure \\
$\m$ & The universal enumerable semi-distribution \\
$K(f)$ & The prefix Kolmogorov complexity of a function $f$ \\
$L$ & An enumeration of program/output pairs defining a machine \\
$\Um$ & The universal monotone machine \\
$\Up$ & The universal prefix machine \\
$\geqt$ & $f(x) \geqt g(x)$ if there exists a $c > 0$ such that  $f(x) > c\cdot g(x)$ for all $x$ \\
$\leqt$ & $f(x) \leqt g(x)$ if there exists a $c > 0$ such that  $f(x) < c\cdot g(x)$ for all $x$ \\
$\eqt$ & $f(x) \eqt g(x)$ if $f(x) \geqt g(x)$ and $f(x) \leqt g(x)$ \\
$x \prefix y$ & $x$ is a prefix of $y$ and $\ell(x) < \ell(y)$ \\
$x \prefixeq y$ & $x$ is a prefix of $y$ \\
\hline
\end{tabular}


\begin{thebibliography}{ABCD}\parskip=0ex

\bibitem[G{\'a}c83]{Gacs83}
Peter G{\'a}cs.
\newblock On the relation between descriptional complexity and algorithmic
  probability.
\newblock {\em Theoretical Computer Science}, 22(1-2):71 -- 93, 1983.

\bibitem[G{\'a}c08]{Gacs08}
Peter G{\'a}cs.
\newblock Expanded and improved proof of the relation between description
  complexity and algorithmic probability.
\newblock {\em Unpublished}, 2008.

\bibitem[HM07]{HM07}
Marcus Hutter and Andrej~A. Muchnik.
\newblock On semimeasures predicting {Martin-L{\"o}f} random sequences.
\newblock {\em Theoretical Computer Science}, 382(3):247--261, 2007.

\bibitem[Hut04]{Hut04}
Marcus Hutter.
\newblock {\em Universal Artificial Intelligence: Sequential Decisions based on
  Algorithmic Probability}.
\newblock Springer, Berlin, 2004.

\bibitem[Hut07]{Hut07}
Marcus Hutter.
\newblock On universal prediction and {B}ayesian confirmation.
\newblock {\em Theoretical Computer Science}, 384(1):33--48, 2007.

\bibitem[Hut09]{Hut09open}
Marcus Hutter.
\newblock Open problems in universal induction \& intelligence.
\newblock {\em Algorithms}, 3(2):879--906, 2009.

\bibitem[LMNT10]{LMNT10}
Steffen Lempp, Joseph Miller, Selwyn Ng, and Dan Turetsky.
\newblock Complexity inequality.
\newblock {\em Unpublished, private communication}, 2010.

\bibitem[LS06]{LR06}
Philip Long and Rocco Servedio.
\newblock Discriminative learning can succeed where generative learning fails.
\newblock In Gabor Lugosi and Hans Simon, editors, {\em Learning Theory},
  volume 4005 of {\em Lecture Notes in Computer Science}, pages 319--334.
  Springer Berlin / Heidelberg, 2006.

\bibitem[LV08]{LV08}
Ming Li and Paul Vitanyi.
\newblock {\em An Introduction to Kolmogorov Complexity and Its Applications}.
\newblock Springer, Verlag, 3rd edition, 2008.

\bibitem[Sol64a]{Sol64a}
Ray Solomonoff.
\newblock A formal theory of inductive inference, {P}art {I}.
\newblock {\em Information and Control}, 7(1):1--22, 1964.

\bibitem[Sol64b]{Sol64b}
Ray Solomonoff.
\newblock A formal theory of inductive inference, {P}art {II}.
\newblock {\em Information and Control}, 7(2):224--254, 1964.

\bibitem[ZL70]{ZL70}
Alexander~K. Zvonkin and Leonid~A. Levin.
\newblock The complexity of finite objects and the development of the concepts
  of information and randomness by means of the theory of algorithms.
\newblock {\em Russian Mathematical Surveys}, 25(6):83, 1970.

\end{thebibliography}
\end{document}